\documentclass[runningheads]{llncs}
\usepackage[T1]{fontenc}
\usepackage{graphicx}
\usepackage{amsmath}
\usepackage{amsfonts}
\usepackage{bm}
\usepackage{multirow}
\usepackage{booktabs} 
\usepackage[misc]{ifsym}

\begin{document}
\title{Understanding Adversarial Robustness of Vision Transformers via Cauchy Problem\footnote{W. Ruan is the corresponding author. This work is supported by Partnership Resource Fund (PRF) on Towards the Accountable and Explainable Learning-enabled Autonomous Robotic Systems from UK EPSRC project on Offshore Robotics for Certification of Assets (ORCA) [EP/R026173/1].}}
\titlerunning{Understanding Adversarial Robustness of Vision Transformers via Cauchy Problem}
\author{Zheng Wang\inst{1} \and
Wenjie Ruan \Letter \inst{1}}
\tocauthor{Zheng Wang (University of Exeter),
Wenjie Ruan (University of Exeter)}
\authorrunning{Z. Wang et al.}

\institute{University of Exeter, Exeter EX4 4PY, UK\\
\email{\{zw360, W.Ruan\}@exeter.ac.uk}}

\toctitle{Understanding Adversarial Robustness of Vision Transformers via Cauchy Problem}

\maketitle              % typeset the header of the contribution
\begin{abstract}
Recent research on the robustness of deep learning has shown that \textit{Vision Transformers (ViTs)} surpass the \textit{Convolutional Neural Networks (CNNs)} under some perturbations, e.g., natural corruption, adversarial attacks, etc. Some papers argue that the superior robustness of ViT comes from the segmentation on its input images; others say that the \textit{Multi-head Self-Attention (MSA)} is the key to preserving the robustness \cite{naseer2021intriguing}. In this paper, we aim to introduce a principled and unified theoretical framework to investigate such argument on ViT's robustness. 
%To better understand the robustness of ViTs, 
We first theoretically prove that, unlike \textit{Transformers} in \textit{Natural Language Processing}, ViTs are \textit{Lipschitz} continuous. Then we theoretically analyze the adversarial robustness of ViTs from the perspective of \textit{Cauchy Problem}, via which we can quantify how the robustness propagates through layers. We demonstrate that the first and last layers are the critical factors to affect the robustness of ViTs. Furthermore, based on our theory, we empirically show that unlike the claims from existing research, MSA only contributes to the adversarial robustness of ViTs under weak adversarial attacks, e.g., \textit{FGSM}, and surprisingly, MSA actually compromises the model's adversarial robustness under stronger attacks, e.g., \textit{PGD attacks}. We release our code via \url{https://github.com/TrustAI/ODE4RobustViT}

\end{abstract}

\section{Introduction}

Since \textit{Transformers} have been transplanted from \textit{Natural Language Processing (NLP)} to \textit{Computer Vision (CV)}, great potential has been revealed by \textit{Vision Transformers} for various CV tasks \cite{khan2021transformers}. It is so successful that some papers even argue that CNNs are just a special case of ViTs \cite{cordonnier2019relationship}. Recently, the robustness of ViTs has been studied, for example, some research showed that ViT has superior robustness than CNNs under natural corruptions \cite{paul2021vision}. Very recently, some researchers have also begun to investigate the robustness of ViTs against \textit{adversarial perturbations} \cite{mahmood2021robustness}.

However, existing research on adversarial robustness for ViTs mainly focuses on \textit{adversarial attacks}. The main idea is to adopt the attacks on CNNs to ViTs, e.g., \textit{SAGA} \cite{mahmood2021robustness} and \textit{IAM-UAP} \cite{hu2021inheritance}. Meanwhile, some pioneering studies demonstrate that ViTs are more robust than CNNs against \textit{adversarial patch attacks}, arguing that the \textit{dynamic receptive field} of MSA is the key factor to its superior robustness \cite{naseer2021intriguing}. On the other hand, some others argue that the tokenization of ViTs plays an essential role in adversarial robustness \cite{aldahdooh2021reveal}. While some researchers say the patch embedding method is a critical factor to contribute the adversarial robustness of ViTs \cite{mao2021towards}. However, most existing works concerning the superior robustness of ViTs are purely based on empirical experiments in an ad-hoc manner. A principled and unified theoretical framework that can quantify the adversarial robustness of ViT is still lacking in the community.

In our paper, instead of analyzing the robustness of \textit{Vision Transformer} purely based on empirical evidence, a theoretical framework has been proposed to examine whether MSA contributes to the robustness of ViTs. Inspired by the fact that ViTs and ResNets share a similar structure of residual additions, we show that, ViTs, under certain assumptions, can be regarded as a \textit{Forward Euler} approximation of the underlying \textit{Ordinary Differential Equations (ODEs)} defined as

\begin{equation*}
    \frac{d\bm{x}}{dt} = \mathcal{F}(\bm{x}, t).
\end{equation*}

With this approximation, each block in transformer can be modeled as the nonlinear function $\mathcal{F}(\bm{x})$. Based on the assumption that function $\mathcal{F}(\bm{x})$ is Lipschitz continuous, we then can theoretically bridge the adversarial robustness with the \textit{Cauchy Problem} by first-order \textit{Taylor expansion} of $\mathcal{F}(\bm{x})$. With the proposed theoretical framework, this paper is able to quantify how robustness is changing among each block in ViTs. We also observe that the first and last layers are vital for the robustness of ViTs.

Furthermore, according to our theoretical and empirical studies, different to the existing claim made by Naseer et al. \cite{naseer2021intriguing} that MSA in ViTs strengthens the robustness of ViTs against patch attacks. We show that MSA in ViTs is {\em not always} improving the model's adversarial robustness. Its strength to enhance the robustness is minimal and even compromises the adversarial robustness against strong $L_p$ norm adversarial attacks. In summary, the key contributions of this paper are listed below.

\begin{itemize}
    \item[1.] To our knowledge, this is the first work to formally bridge the gap between the robustness problem of ViTs and the \textit{Cauchy problem}, which provides a principled and unified theoretical framework to quantify the robustness of transformers.
    
    \item[2.] We theoretically prove that ViTs are Lipschitz continuous on vision tasks, which is an important requisite to building our theoretical framework. 

    \item[3.] Based on our proposed framework, we observe that the first and last layers in the encoder of ViTs are the most critical factors to affect the robustness of the transformers.
    
    \item[4.] Unlike existing claims, surprisingly, we observe that MSA can only improve the robustness of ViTs under weak attacks, e.g., {\em FGSM attack}, and it even compromises the robustness of ViTs under strong attacks, e.g., {\em PGD attack}. 
\end{itemize}

\section{Related Work}
\subsection{Vision Transformers and Its Variants}

To the best of our knowledge, the first work using the transformer to deal with computer vision tasks is done by Carion et al. \cite{carion2020end}, since then, it has quickly become a research hotspot, though it has to be pre-trained on a larger dataset to achieve comparable performance due to its high complexity and lack of ability to encode local information. To reduce the model complexity, \textit{DeiT} \cite{touvron2021training} leverages the \textit{Knowledge Distillation} \cite{hinton2015distilling} technique, incorporating information learned by \textit{Resnets} \cite{he2015deep}; \textit{PvT} \cite{wang2021pyramid} and \textit{BoTNet} \cite{srinivas2021bottleneck} adopt more efficient backbones; \textit{Swim Transformer} \cite{liu2103swin} and \textit{DeepViT} \cite{zhou2021deepvit} modifies the MSA. Other variants, e.g., \textit{TNT}, \textit{T2T-ViT}, \textit{CvT}, \textit{LocalViT} and \textit{CeiT} manage to incorporate local information to the ViTs \cite{khan2021transformers}.

\subsection{Robustness of Vision Transformer}

Many researchers focus on the robustness of ViTs against \textit{natural corruptions} \cite{hendrycks2019natural} and empirically show that ViTs are more robust than CNNs \cite{paul2021vision}. The adversarial robustness of ViTs has also been empirically investigated. Compared with CNNs and MLP-Mixers under different attacks, it claims that for most of the white-box attacks, some black-box attacks, and Universal Adversarial Perturbations (UAPs) \cite{moosavi2017universal}, ViTs show superior robustness \cite{naseer2021intriguing}. However, ViTs are more vulnerable to simple FGSM attacks \cite{bhojanapalli2021understanding}. The robustness of variants of ViTs is also investigated and shown that the local window structure in Swim-ViT harms the robustness and argues that positional embedding and the \textit{completeness/compactness} of heads are crucial for performance and robustness \cite{mao2021discrete}.

However, the reason for the superior robustness of ViTs is rarely investigated. Most of the research concentrate on frequency analysis \cite{paul2021vision}. Benz et al. argue that shift-invariance property \cite{benz2021adversarial} harms the robustness of CNNs. Naseer et al. say the flexible receptive field is the key to learning more shape information which strengthens the robustness of ViTs by studying the severe occlusions \cite{naseer2021intriguing}. And yet Mao et al. argue that ViTs are still overly reliant on the texture, which could harm their robustness against out-of-distribution data \cite{mao2021discrete}. Qin et al. investigate the robustness from the perspective of robust features and argue ViTs are insensitive to patch-level transformation, which is considered as non-robust features~\cite{qin2021understanding}.

\subsection{Deep Neural Network via Dynamic Point of View}

The connection between differential equations and neural networks is first introduced by S. Grossberg \cite{grossberg2013recurrent} to describe a continuous additive RNN model. After ResNet had been proposed, new relations appeared that regard forward prorogation as Euler discretization of the underlying ODEs \cite{ruseckas2019differential}. And many variants of ResNets can also be analyzed in the framework of numerical schemes for ODEs, e.g., PolyNet, FracalNet, RevNet and LMResNet \cite{lu2018beyond}. Instead of regarding neural networks as discrete methods, Neural ODE has been proposed \cite{chen2018neural}, replacing the ResNet with its Underlying ODEs, and the parameters are calculated by a black-box ODE solver. However, E. Dupont et al. \cite{dupont2019augmented} argue that neural ODEs hardly learn some representations. In addition to ODEs, PDEs and even SDEs are also involved in analyzing the Neural Network \cite{sun2018stochastic}.

\section{Preliminaries}

The original ViTs are generally composed of \textit{Patch Embedding}, \textit{Transformer Block} and \textit{Classification Head}. We follow the definition from \cite{dosovitskiy2020image}. Let $\bm{x} \in \mathbf{R}^{H\times W\times C}$ stands for the input image. Hence, Each image is divided equally into a sequence of $N=HW/P^2$ patches, and each one is denoted as $\bm{x}_p \in \mathbf{R}^{N\times (P^2 \cdot C)}$. 

\begin{align*}
    \bm{z}_0 &= [\bm{x}_{class},\bm{x}_p^1\bm{E},\bm{x}_p^2\bm{E},...,\bm{x}_p^N\bm{E}]+\bm{E}_{pos}, \\ 
    \bm{z}_{l}^{'} &= MSA(LN(\bm{z}_{l-1})) + \bm{z}_{l-1}, \\
    \bm{z}_{l} &= MLP(LN(\bm{z}_{l}^{'})) + \bm{z}_{l}^{'}, \\
    \bm{y} &= LN(\bm{z}_L^0), 
\end{align*}
where $\bm{E} \in \mathbb{R}^{P^2 \cdot C \times D}, \bm{E}_{pos} \in \mathbb{R}^{(N+1) \times D}$ and $l = 1,2,...,L$. $LN$ denotes \textit{Layer Normalization}, \textit{MSA} is \textit{Multihead Self-Attention} and \textit{MLP} represents \textit{Multilayer Perceptron}. MSA is the concatenation of \textit{Self-Attentions (SA)} before linear transformation by $W^{(O)} \in \mathbb{R}^{D\times D}$ defined by 
\begin{align*}
    MHA := \begin{pmatrix} SA_1 & SA_2 &...& SA_H \end{pmatrix} W^{(O)}, 
\end{align*}
where $H$ is the number of heads and SA is defined by
\begin{align*}
    SA := P\bm{z}W^{(V)} = softmax\bigg(\bm{z}W^{(Q)}W^{(K)T}\bm{z}^T \bigg)W^{(V)}, 
\end{align*}
where $W^{(Q)}, W^{(K)}, W^{(V)} \in \mathbb{R}^{D \times (D/H)}$, and $\bm{z} \in \mathbb{R}^{N \times D}$ defines the inputs of transformers.  
 
\section{Theoretically Analysis}

\subsection{Vision Transformers are Lipschitz}

To model the adversarial robustness to Cauchy Problem, we first prove that ViTs are Lipschitz functions. Unlike the conclusion drawn by Kim et al. \cite{kim2021lipschitz} that Dot-product self-attention is not Lipschitz, it can be proved that Vision Transformers are Lipchitz continuous since inputs are bounded between $[0, 1]$. We follow the same definition from \cite{kim2021lipschitz} that a function $f: \mathcal{X} \to \mathcal{Y}$ is called Lipschitz continuous if $\exists K \geq 0$ such that $\forall \bm{x}\in \mathcal{X}, \bm{y}\in \mathcal{Y}$ we have

\begin{equation}
    d_{\mathcal{Y}}(f(\bm{x}), f(\bm{x}_0)) \leq K d_{\mathcal{X}}(\bm{x}, \bm{x}_0),
\end{equation}
where $(\mathcal{X}, d_{\mathcal{X}}), (\mathcal{Y}, d_{\mathcal{Y}})$ are given metric spaces, and given $p$-norm distance, the Lipschitz constant $K$ is given by 
\begin{equation}
    Lip_p(f) = \sup_{\bm{x} \neq \bm{x}_0}\frac{\lVert f(\bm{x}) - f(\bm{x}_0)\rVert_p}{\lVert \bm{x} - \bm{x}_0 \rVert_p}.
\end{equation}
Similar to the analysis by Kim et al. \cite{kim2020torchattacks}, since Linear transformation by $W^{(V)}$ is Lipchitz and does not impact our analysis, we will drop it and focus on the non-linear part of $P\bm{z}$.   

Since Patch embeddings are conducted by convolutional operations and the classification heads are fully connected layers, they are Lipchitz continuous \cite{kim2020torchattacks}. Therefore as long as the transformer blocks are Lipschitz continuous, ViTs are Lipschitz continuous because the composite Lipchitz functions, i.e., $f\circ g$, are also Lipschitz continuous \cite{federer2014geometric}. To this end, we have the Theorem \eqref{thm_lip}.

\begin{theorem}\label{thm_lip}
(\textit{Transformer Blocks in ViTs are Lipschitz continuous})

Given vision transformer block with trained parameters $\bm{w}$ and convex bounded domain $\mathcal{Z}_{l-1} \subseteq \mathbb{R}^{N \times D}$, we show that the transformer block $\mathcal{F}_l: \mathcal{Z}_{l-1} \to \mathbb{R}^{N\times D}$ mapping from $z_{l-1}$ to $z_l$ is Lipchitz function for all $l = 1,2,...,L$. 
\end{theorem}
\begin{proof}
For simplicity, we only prove the case that the number of heads $H$ and the dimension of patch embedding $D$ are all equal to $1$. The general case can be found in Appendix.    

Because the composition of the transformer block includes an MLP layer that is Lipchitz continuous, as argued by Kim et al. \cite{kim2020torchattacks}, it is the non-linear part of MSA that need to be proved Lipchitz continuous. We formulated the non-linear part as mapping $f:\mathcal{Z} \to \mathbb{R}^{N\times 1}$ shown in Equation \eqref{eq:MSA} 
\begin{equation}
    \begin{split}
    f(\bm{z}) &= softmax(a\bm{zz}^T)\bm{z} = P\bm{z}
    = \begin{pmatrix} 
    p_1(\bm{u}_1) & \cdots & p_N(\bm{u}_1)  \\ 
    \vdots & \ddots &\vdots \\
    p_1(\bm{u}_N) & \cdots & p_N(\bm{u}_N) \\
    \end{pmatrix}\bm{z} \\ \label{eq:MSA}            
    \end{split}
\end{equation}
where $a = W^{(Q)}W^{(K)T} \in \mathbb{R}$, $\bm{z} \in \mathcal{Z}$ which is a bounded convex set and belongs to $\mathbb{R}^{N\times 1}$, $P$ is defined by \textit{softmax} operator. Each row in $P$ defines a discrete probability distribution. Therefore $P$ can be regarded as the \textit{transition matrix} for a finite discrete \textit{Markov Chain} with ${z_1,...,z_n}$ as observed value for random valuables. Since $f$ has continuous first deviates, \textit{Mean Value Inequality} can be used to find Lipchitz constant. Let $\bm{z}, \bm{z}_0 \in \mathcal{Z}$ and $\lVert \cdot \rVert_p$ denote the $p$-norm distance for vectors and \textit{induced norm} for matrices. Specifically, when $p = 2$, the \textit{induce norm} coincides with \textit{spectral norm}, then we have 
\begin{align}
    \lVert f(\bm{z}) - f(\bm{z}_0) \rVert_2 \leq \lVert J_f(\bm{\xi}) \rVert_2 \lVert(\bm{z} - \bm{z}_0)\rVert_2,  \label{eq:lemma1_MVT}
\end{align}
where $\xi \in \mathcal{Z}$ is on the line through $\bm{x}$ and $\bm{x}_0$, and $J_f(\cdot)$ denotes the \textit{Jacobian} of $f$. As long as the Jacobian $J_f$ is bounded for $\mathcal{Z}$, $f$ is Lipschitz continuous. The Jacobian $J_f$ is shown in Equation \eqref{eq:lemma1_J} (see detail in Appendix).    
\begin{equation}
\begin{split}
        J_f(\bm{z}) &= a\big\{diag(\bm{z})Pdiag(\bm{z}) - Pdiag(\bm{z})diag(\bm{\mu}) + diag(\bm{\sigma}^2) \big\} + P, \label{eq:lemma1_J}
\end{split}    
\end{equation}

where $\bm{\mu} = P\bm{z}$ define the mean vector for the Finite Markov Chain and the \textit{variance} are defined by 
\begin{equation}
\bm{\sigma}^2 = \begin{pmatrix}
                \sum_{k=1}^{N}p_k(\bm{u}_1)x_k^2 - \big(\sum_{k=1}^{N}p_k(\bm{u}_1)x_k\big)^2 \\
                \vdots \\
                \sum_{k=1}^{N}p_k(\bm{u}_N)x_k^2 - \big(\sum_{k=1}^{N}p_k(\bm{u}_N)x_k\big)^2 \\
\end{pmatrix} = \begin{pmatrix}
                \sigma_1^2 \\
                \vdots \\
                \sigma_N^2
\end{pmatrix},    
\end{equation}
Since every component on the right-hand-side in \eqref{eq:lemma1_J} is bounded since $\bm{z}$ is bounded. We conclude that $J_f(\bm{z})$ is also bounded, therefore the Lipchitz continuous. 
\end{proof}
\begin{remark}
The use of the Mean Value Theorem requires the domain $\mathcal{Z}$ to be convex, however as long as $\mathcal{Z}$ is bounded, we can always find a larger convex set $\mathcal{Z'} \supseteq \mathcal{Z}$.     
\end{remark}

\begin{remark}
Different from the conclusion drawn by Kim et al. \cite{kim2020torchattacks} that the transformer is not Lipschitz continuous, ViTs are Lipschitz continuous due to the bounded input. 
\end{remark}

\subsection{Model Adversarial Robustness as Cauchy Problem}

Since there exists the \textit{Residual Structure} in the \textit{Transformer Encoder}, just like \textit{ResNet}, which can be formulated as \textit{Euler Method} \cite{lu2018beyond}, the \textit{forward propagation} through Transformer Encoder can also be regarded as a \textit{Forward Euler Method} to approximate the underlying \textit{Ordinary Differential Equations (ODEs)}. 

Let $f: \mathcal{X} \to \mathcal{Y}$ denote the ViTs, where $\mathcal{X} \subseteq \mathbb{R}^n$ denotes the input space and $\mathcal{Y} = \{1,2,...,C\}$ refers to the labels, and $\mathcal{F}_i, i=1,...,L$ denote the basic blocks. Notice that for simplicity, let $\mathcal{F}_1(\bm{x}_0; \bm{w}_0) + \bm{x}_0$ refer to the patch embedding and $\mathcal{F}_L(\bm{x}_{L-1}; \bm{w}_{L-1}) + x_{L-1}$ be the classification head, the rest are transformer blocks. Hence, the forward propagation can be described in Equation \eqref{eq:resnetlike}. 

\begin{equation}
\begin{cases}
    \bm{x}_{k} &= \mathcal{F}_{k}(\bm{x}_{k-1}; \bm{w}_{k-1}) + \bm{x}_{k-1}, k=1,...,L  \\
    \bm{y}_{logit} &= softmax(LP(\bm{x}_{L})) \\
    \bm{y} &= \arg\max_{\mathcal{Y}} \bm{y}_{logit}  ,
    \label{eq:resnetlike}
\end{cases}
\end{equation}
where $\bm{x}_{0} \in \mathcal{X}$, $LP(\cdot)$ stands for \textit{Linear Projection}, $\bm{y}_{logit}$ shows the likelihood for each class and $\bm{y} \in \mathcal{Y}$ denote the classification result. As argued by Liao, et al. \cite{liao2016bridging}, the Transformer blocks in Equation \eqref{eq:resnetlike} can be regarded as \textit{Forward Euler} approximation of the underlying ODE shown below.

\begin{equation}
    \frac{d}{dt}\bm{x}(t) = \mathcal{F}(\bm{x}, t), t \in [t_0, T]
    \label{eq:resnetlikeODE}
\end{equation}
where $\mathcal{F}(\cdot)$ corresponds to the basic blocks in ViTs and $t \in [t_0, T]$ refers to the continuous indexing of those blocks.      

The backward-propagation of Equation \eqref{eq:resnetlikeODE} can be regarded as an estimation problem for parameters $\bm{w}$ of given boundary conditions defined by $\mathcal{X}$ and $\mathcal{Y}$, which leads to Neural ODEs \cite{chen2018neural}.   
   
Before the main theorem that models the adversarial robustness as Cauchy problem, we first define the adversarial robustness metrics. Given neural network $f$, and the fixed input $\bm{x} \in \mathcal{X}$, the \textit{local Adversarial Robustness} proposed by Bastani et al. \cite{bastani2016measuring} is defined as

\begin{equation*}
    \rho(f, \bm{x}) \overset{def}{=} \inf\{ \epsilon > 0 | \exists \hat{\bm{x}}: \lVert \hat{\bm{x}}-\bm{x} \rVert \leq \epsilon, f(\hat{\bm{x}}) \neq f(\bm{x}) \},
\end{equation*}

where $\lVert \cdot \rVert$ defines the general $L_p$ norm. Usually, $p$ is taken as $1,2$ and $\infty$. The adversarial robustness is defined as the minimum radius that the classifier can be perturbed from their original corrected result. As illustrated in figure \eqref{fig:illustration_rho}, considering the the fact that even in the final laryer $\Delta \bm{x}(T)_1 < \Delta \bm{x}(T)_2$, it is still possible that $softmax(LP(\bm{x}(T) + \Delta \bm{x}(T)_1))$ has been perturbed but $softmax(LP(\bm{x}(T) + \Delta \bm{x}(T)_2))$ is not, we use the minimal distortion to define the robustness as 

\begin{equation} \label{eq:robustmetric_star}
    \rho^{\star}(f,\bm{x}) \overset{def}{=} \inf_{\lVert \hat{\bm{x}}(T) - \bm{x}(T) \rVert}\rho(f, \bm{x}),
\end{equation}
where $\hat{\bm{x}}(T) - \bm{x}(T) = \Delta \bm{x}(T)$. 

\begin{figure}[!t]
    \centering
    \includegraphics[width=0.8\textwidth]{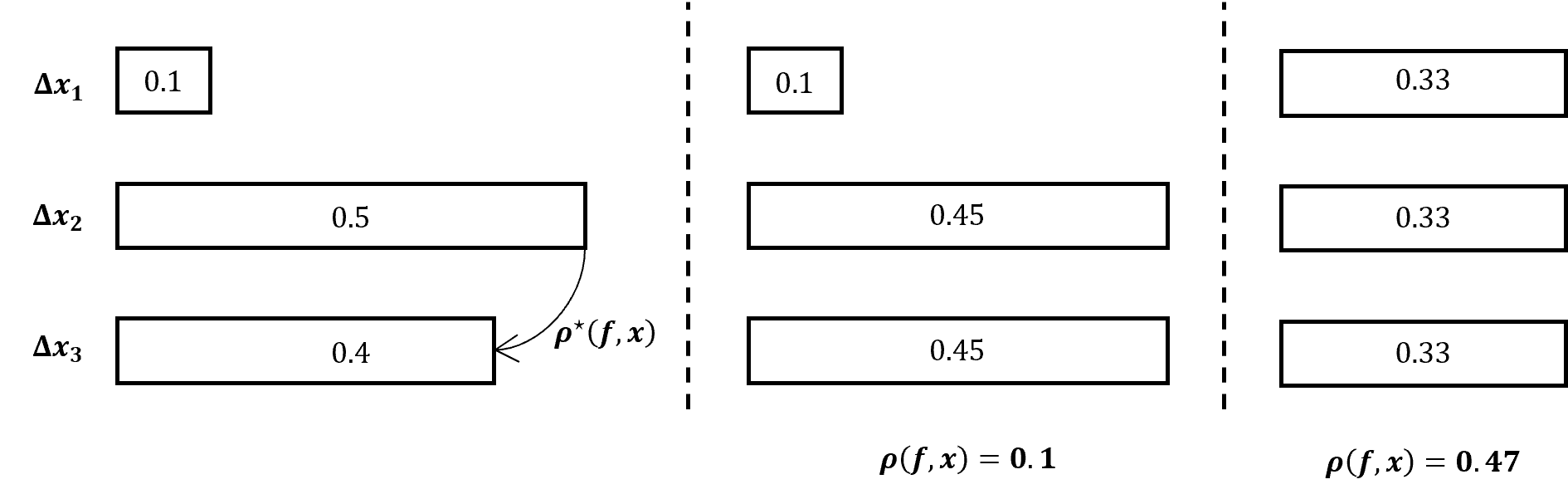}
    \caption{Illustration of $\rho^{\star}(f,x)$. For better illustration $L_1$-norm is taken while calculating the $\rho(f,x)$.}
    \label{fig:illustration_rho}
\end{figure}

\begin{lemma} \label{lemma_solution}
(\textit{Existence and Uniqueness for the Solution of Underlying ODE}) Since the continuous mapping $\mathcal{F}$ defined in ODE \eqref{eq:resnetlikeODE} satisfies the Lipschitz condition on $\bm{z} \in \mathcal{Z}$ for $t \in [t_0, T]$ as claimed in Theorem \eqref{thm_lip}, where $\mathcal{Z}$ is a bounded closed convex set. There exists and only exists one solution for the underlying ODE defined in \eqref{eq:resnetlikeODE}.  
\end{lemma}

\begin{lemma} \label{lemma_approx}
(\textit{Error Bound for Forward Euler approximation}) Given Forward Euler approximation shown in Equation \eqref{eq:resnetlike} and its underlying ODE in Equation \eqref{eq:resnetlikeODE}. Let $K > 0$ denotes the Lipschitz constant for the underlying ODE, and $\lVert \hat{\mathcal{F}}(x,t) - \mathcal{F}(x,t) \rVert \leq \delta$, hence the error of solution is given by 
\begin{equation*}
    \lVert \Delta \bm{x} \rVert \leq \frac{\delta}{K}(e^{K|t - t_0|}-1)
\end{equation*}
\end{lemma}

Since $\mathcal{F}(x,t)$ is continuous, $\delta$ can be arbitrary small as long as step for Euler approximation is small enough, namely, neural network is deep enough. The proof of lemma \eqref{lemma_solution} and \eqref{lemma_approx} can be found in \cite{coddington1955theory}.

\begin{theorem} \label{thm_main}   
Let $f$ and $g$ be two neural networks defined in Equation \eqref{eq:resnetlike}, which have the underlying ODEs as shown in Equation \eqref{eq:resnetlikeODE}, and denote the basic blocks of $g$ as $\mathcal{G}_k, k=1,...,L^{'}$ with its corresponding ODE defined as $\mathcal{G}$ to show the difference. Given point $\bm{x} \in \mathcal{X}$ and robustness metric $\rho^{\star}(\cdot)$ defined in \eqref{eq:robustmetric_star}, classifier $f$ is more robust than $g$, such that
\begin{equation}
    \rho^{\star}(f, \bm{x}) \leq \rho^{\star}(g, \bm{x}) \label{eq:thm1_1},     
\end{equation}

if $\forall t \in [t_0, T]$
\begin{equation}
    \sigma_{max}(J_{\mathcal{F}}(t)) \leq 
    \sigma_{max}(J_{\mathcal{G}}(t)) \label{eq:thm1_2}
\end{equation}

where $J_f(t)$ and $J_g(t)$ refers to the Jacobian of the basic blocks $\mathcal{F}$ and $\mathcal{G}$ w.r.t. $\bm{x}$ and $\sigma_{max}(\cdot)$ denotes the \textit{largest singular value} .   
\end{theorem}
\begin{proof}

Consider 2 solutions $\bm{x}(t), \bm{\hat{x}}(t)$ of ODE defined in \eqref{eq:resnetlikeODE} such that for $\epsilon > 0$
\begin{equation*}
    \lVert \bm{\hat{x}}(t_0) - \bm{x}(t_0) \rVert_2 \leq \epsilon
\end{equation*}
and let $\Delta \bm{x}(t) =  \bm{\hat{x}}(t) - \bm{x}(t), t \in [t_0, T]$ hence 
\begin{equation}
\begin{split}
    \frac{d}{dt}\Delta \bm{x} &= \mathcal{F}(\bm{\hat{x}}, t) - \mathcal{F}(\bm{x},t) = J_{\mathcal{F}}(t)\Delta \bm{x} + \bm{r}_{\mathcal{F}}(\Delta \bm{x}), 
\end{split} \label{eq:thm1_4}
\end{equation}
where $\bm{r}_{\mathcal{F}}(\Delta \bm{x})$ is the residual of \textit{Taylor Expansion} of $\mathcal{F}$ w.r.t. $\bm{x}$, such that $\lVert \bm{r}_{\mathcal{F}}(\Delta \bm{x}) \rVert = \mathcal{O}( \lVert \Delta \bm{x} \rVert^2)$ \cite{ascher1995numerical}. Instead of $\Delta \bm{x}$, $\lVert \Delta \bm{x} \rVert_2$ is more of our interest, hence 
\begin{equation}
\begin{split}
    \frac{d}{dt}\lVert \Delta \bm{x} \rVert_2 \leq \lVert \frac{d}{dt} \Delta \bm{x} \rVert_2 \leq \lVert J_{\mathcal{F}}(t)\rVert_2 \lVert \Delta\bm{x}\rVert_2 + \mathcal{O}(\lVert \Delta \bm{x} \rVert_2^2), 
\end{split} \label{eq:thm1_5}
\end{equation}
since $\lVert \Delta \bm{x}(t_0) \rVert_2 = 0$ is trivial, we assume $\lVert \Delta \bm{x}(t_0) \rVert_2 > 0$. And because there exist unique solution for the ODE system, we have $\lVert \Delta \bm{x}(t) \rVert_2 > 0, t \in[t_0, T]$ therefore Equation \eqref{eq:thm1_5} becomes 
\begin{equation}
\begin{split}
    \frac{1}{\lVert \Delta\bm{x}\rVert_2}\frac{d}{dt}\lVert \Delta \bm{x} \rVert_2 &\leq \lVert J_{\mathcal{F}}(t)\rVert_2 + \mathcal{O}(\lVert \Delta \bm{x} \rVert_2). 
\end{split} \label{eq:thm1_6}
\end{equation}
After integral of the both sides from $t_0$ to $T$ we have 
\begin{equation*}
\begin{split}
    \int_{t_0}^{T}\frac{1}{\lVert \Delta\bm{x}\rVert_2}d\lVert \Delta \bm{x} \rVert_2 \leq \int_{t_0}^{T} \lVert J_{\mathcal{F}}(t)\rVert_2 + M\epsilon dt, 
\end{split}
\end{equation*}
where $M > 0$ is a given large number. The integral for $[t_0, T]$ is given by 
\begin{equation}
    \lVert \Delta \bm{x}(T) \lVert_2 \leq \epsilon e^{\int_{t_0}^{T} \lVert J_{\mathcal{F}}(t)\rVert_2dt + (T-t_0)M\epsilon}. \label{eq:thm1_8}
\end{equation}
It is obvious that the perturbed output of neural network $\Delta \bm{x}(T)$ is actually bounded by the right-hand-side of Equation \eqref{eq:thm1_8} which is determined by the $\lVert J_{\mathcal{F}}(t) \rVert_2, t\in [t_0, T]$, namely the largest \textit{singular value} of $J_{\mathcal{F}}(t)$, denoted as $\sigma_{max}(J_{\mathcal{F}}(t))$. The rest of the proof is simple, since if $\forall t \in [t_0, T]$ \eqref{eq:thm1_2} holds and $(T-t_0)M\epsilon$ is negligible, we have    
\begin{equation*}
    \lVert \Delta \bm{x}_{\mathcal{F}}(T) \lVert_2 \leq \epsilon e^{\int_{t_0}^{T} \lVert J_{\mathcal{F}}(t)\rVert_2dt} \leq \epsilon e^{\int_{t_0}^{T} \lVert J_{\mathcal{G}}(t)\rVert_2dt},
\end{equation*}
therefore for any $\lVert \Delta \bm{x}_{\mathcal{F}}(t_0) \rVert_2 \leq \rho^{\star}(g, \bm{x})$ the classification result will also not change for $f$, hence the Equation \eqref{eq:thm1_1}. 
\end{proof}
\begin{remark}
Theorem \eqref{thm_main} is particularly useful for adversarial perturbation since the approximation in Equation \eqref{eq:thm1_8} relies on the narrowness of $\epsilon$. If it is too large, the first-order approximation may fail.
\end{remark}

\begin{remark}
Theorem \eqref{thm_main} assumes that the approximation error induced in lemma \eqref{lemma_approx} is small enough to neglect. For very shallow models, e.g., ViT-$S_1$, ViT-$S_2$, the relation is violated, as is shown in Table \eqref{tab:attacks_singular}. 
\end{remark}

\section{Empirical Study}

In order to verify the proposed theorem and find out whether self-attention indeed contributes to the adversarial robustness of ViTs, we replace the self-attention with a 1-D convolutional layer, as shown in figure \eqref{fig:vit_and_covit}. And we name the modified model \textit{CoViT}, which stands for \textit{Convolutional Vision Transformer}. We use Average Pooling  instead of the classification token since the classification token can only learn the nearest few features rather than the whole feature maps for CoViT.

Both ViTs and CoViTs are trained from sketches without any pertaining to ensure that they are comparable. \textit{Sharpness-Aware Minimization (SAM)} \cite{foret2020sharpness} optimizer is used throughout the experiments to ensure adequate clear accurate. 

\begin{figure}[t]
\centering
\includegraphics[width=0.8\textwidth]{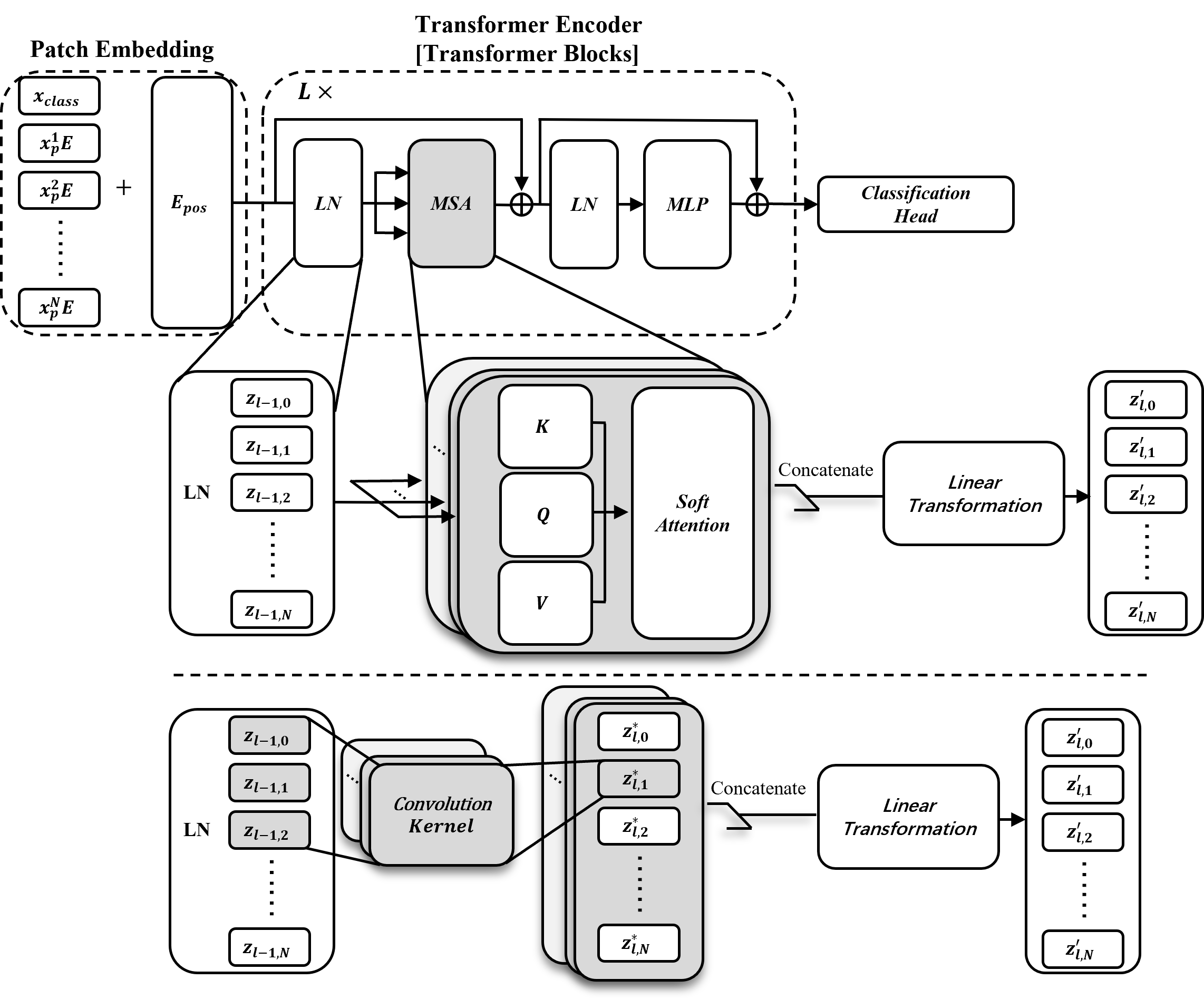}
\caption{Illustration of ViT and CoViTs. After \textit{Patch Embedding}, the \textit{Transformer Encoder} is composed of $L \times$ \textit{Transformer Blocks}, of which in each \textit{K,Q} and \textit{V} stands for \textit{Key, Query} and \textit{Value} are computed as linear projection from former tokens $z_{l-1}$, hence \textit{Self-Attention} is calculated as $softmax(\frac{QK^T}{\sqrt{D}})V$. In order to better understand whether self-attention indeed contributes to adversarial robustness, it is replaced by \textit{1-D convolution layers} where different kernel are used and the intermediates, denoted by $z_l^{*}$, are generated before concatenation and linearly projecting to $z_l^{'}$. The kernel size can be different for each convolutional projection.} \label{fig:vit_and_covit}
\end{figure}

\subsection{Configuration and Training Result}
The configurations for Both ViTs and CoViTs have an input resolution of 224 and embedding sizes of 128 and 512. The use of a smaller embedding size of 128 is to calculate the maximum singular value exactly. An upper bound is calculated instead for models with a larger embedding size since the exact calculation is intractable. We change the number of heads for ViTs, the kernels for CoViTs, the depth, and the patch size for the experiment. All the models are divided into four groups: S, M, L, T, standing for \textit{Small}, \textit{Medium}, \textit{Large}, and \textit{Tiny} of parameter size. The tiny model uses an embedding size of 128. The detailed configuration is shown in Appendix. 

All the models are trained on CIFAR10, and the base optimizer for SAM is SGD with the One-cycle learning scheduler of maximum learning rate equals to $0.1$. In order to have a better performance, augmentations, including \textit{Horizontal Flipping, Random Corp} and \textit{Color Jitter}, are involved during training. We resize the image size to the resolution of $224 \times 224$. The model with an embedding size of $512$ is trained by $150$ epochs, and the tiny model with an embedding size of $128$ is trained by $300$ to achieve adequate performance. The performance of models within the same group is similar, and the shallow networks, e.g., ViT-$S_1$, ViT-$S_2$, CoViT-$S_1$, CoViT-$S_2$, are harder to train and may need extra training to be converged. This may be due to the optimizer used, i.e., SAM, since SAM will try to find shallow-wide optima instead of a deep-narrow one, which requires a stronger model capacity.   

The experiments are conducted on Nvidia RTX3090 with python 3.9.7, and realized by PyTorch 1.9.1. \textit{Torchattacks} \cite{kim2020torchattacks} is used for adversarial attacks. 

\subsection{Study for Small Scale Models}
In order to find out whether MSA contributes to the adversarial robustness of ViTs and verify Theorem \eqref{thm_main}, tiny models with an embedding size of $128$ are employed and attacked by $L_2$-norm PGD-20 and CW. The threshed of successful attacks for CW is set to $260$. The corresponding average and standard deviation of the exact maximum singular value for the Jacobian is calculated over layers and images to indicate the overall magnitude of $\sigma_{max}(t)$ over the interval $[t_0, T]$. In other words, we calculate the mean value of $\int_{t_0}^{T}\lVert J_{\mathcal{F}}(t) \rVert_2 dt$ for 500 images to indicate the global robustness of the classifier. The PGD-20($L_2$) and CW share the same setting with large-scale experiments in Table \eqref{tab:attacks_singular}, except that the total iteration for PGD is $20$ instead of $7$.   

{\bf Verification of Theorem~~}
The result, as shown in Table \eqref{tab:exact_sigma}, generally matches our theoretic analysis since the most robust model has the lowest average maximum singular value. It is worth mentioning that the smaller value of $\bar{\sigma}_{max}$ cannot guarantee stronger robustness for ViTs in Table \eqref{tab:exact_sigma}, since the standard deviations of $\sigma_{max}$ are much larger, e.g., $11.66$, than that of CoViTs. 

{\bf Contribution of MSA~~}
Another observation is that CoViTs are generally more robust than ViTs. In other words, without enough embedding capacity, Self-Attention could even hurt both the robustness and generalization power. In addition, increasing the models' depth will enhance both generalization power and robustness.

\begin{table*}[!t]
\caption{Attack result and the average maximum singular value for tiny model. All the models have embedding size of 128 with different depth and head or kernels. The cleaning accuracy and the robust accuracy for PGD-20 and CW attack are shown in the Table. The mean of exact maximum singular value $\sigma_{max}$ for over layers and 500 input images are calculated with standard deviation shown in square. The highest accuracy and lowest maximum singular value are marked in bold.}
\label{tab:exact_sigma}
\vskip 0.15in
\begin{center}
\begin{small}
\begin{sc}
\begin{tabular}{l|l|l|l|l|l|l}
\toprule
 \multirow{2}{*}{Net-name} & \multirow{2}{*}{Depth} &\#Head/ & \multirow{2}{*}{Clear Acc.} & \multirow{2}{*}{PGD-20($L_{2}$)} & \multirow{2}{*}{CW($L_2$)} &  \multirow{2}{*}{$\Bar{\sigma}_{max}$} \\
 & &Kernel & & & & \\
\midrule
ViT-$T_1$   & $4$ & $1$          &$0.819$ & $0.397$      & $0.0305$ & $10.45(4.125)$\\
ViT-$T_2$   & $4$ & $4$          &$0.820$ & $0.407$      & $0.039$ & $18.18(11.66)$\\
CoViT-$T_1$ & $4$ & $K3$         &$0.849$ & $0.492$      & $\bm{0.083}$ & $9.25(0.822)$\\
CoViT-$T_2$ & $4$ & $4\times K3$ &$\bm{0.852}$ & $\bm{0.492}$ & $0.036$ & $\bm{9.186(1.088)}$\\
\midrule
ViT-$T_3$   & $8$ & $1$          &$0.836$ & $0.442$      & $0.040$ &  $7.17(1.52)$\\
ViT-$T_4$   & $8$ & $4$          &$0.834$ & $0.463$      & $0.048$ & $10.03(2.73)$ \\
CoViT-$T_3$ & $8$ & $K3$         &$\bm{0.860}$ & $0.514$      & $\bm{0.076}$ & $\bm{6.413(0.562)}$\\
CoViT-$T_4$ & $8$ & $4\times K3$ &$0.860$ & $\bm{0.515}$ & $0.065$ & $6.662(0.386)$\\
\bottomrule
\end{tabular}
\end{sc}
\end{small}
\end{center}
\vskip -0.1in
\end{table*}

\begin{figure}[!t]
    \centering
    \includegraphics[width=1\textwidth]{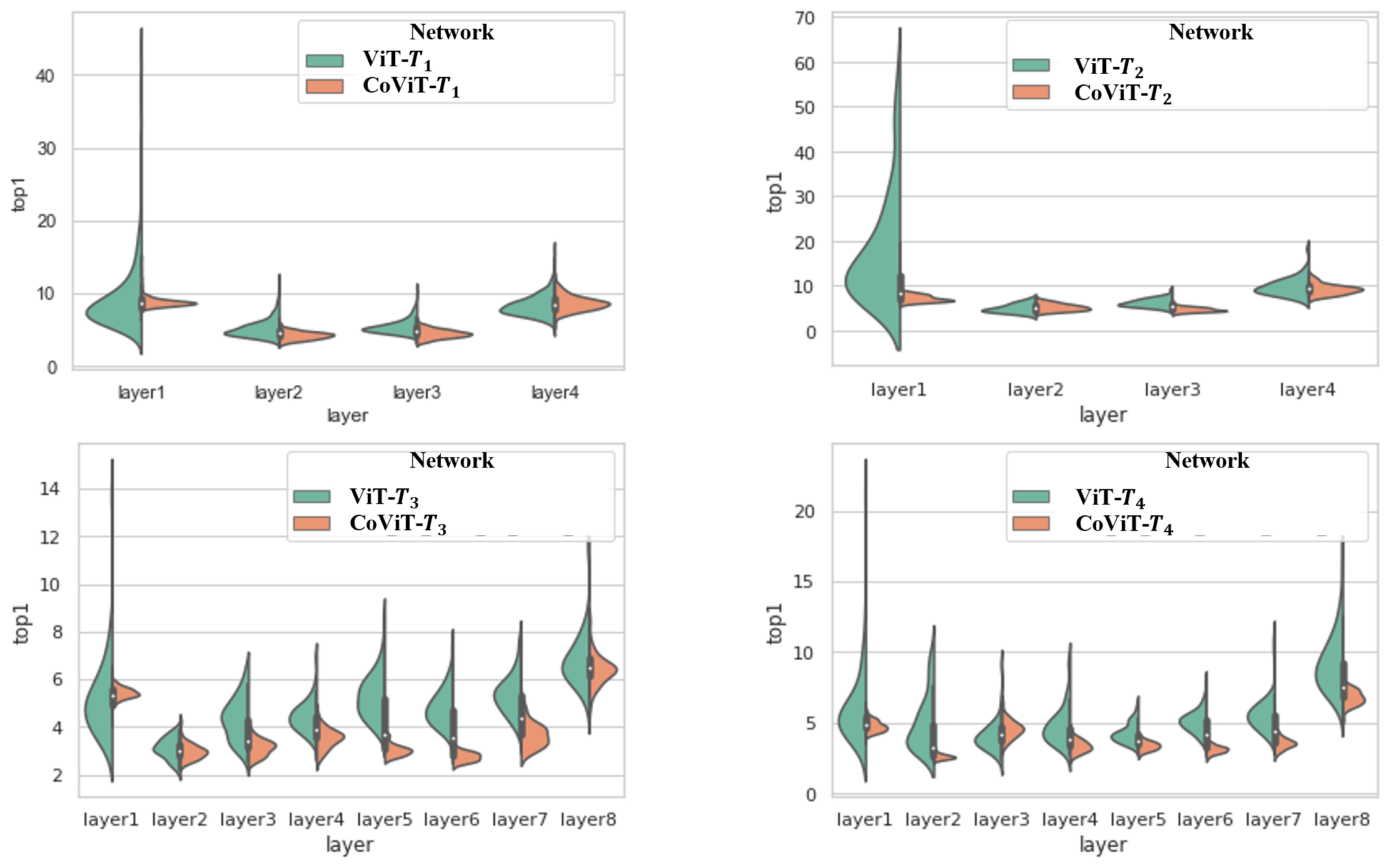}
    \caption{Violin plot for maximum singular value for each layer of the ViTs/CoViTs. The y-axis shows the maximum singular value.}
    \label{fig:exact_sigma}
\end{figure}

{\bf Distribution of Maximum Singular Value in Each Layer~~}
In order to know which layer contributes most to the  non-robustness, the distribution of $\sigma_{max}$ is calculated. The layer that has the highest value of $\sigma_{max}$ may dominate the robustness of the network. As is shown in figure \eqref{fig:exact_sigma}, maximum singular values for the CoViTs are much concentrated around the means, reflecting more stable results for classification. And the maximum singular values for the first and last layer of all tiny models are significantly higher than that for in-between layers, indicating that the first and last layers in the transformer encoder are crucial for adversarial robustness. 

\subsection{Contribution of MSA to Robustness for Large Scale Models}
We attack both ViTs and CoViTs with FGSM, PGD, and CW for large-scale models and compare the robust accuracy. And since it is intractable to compute exact maximum singular value for the matrix of size $(128\cdot 512) \times (128\cdot 512)$, an upper bound of maximum singular value is calculated as 
\begin{equation}\label{eq:single_value_approx}
    \lVert J \rVert_2 \leq \bigg(\lVert J \rVert_1 \lVert J \rVert_{\infty}\bigg)^{\frac{1}{2}},
\end{equation}
which is used as an approximation to the maximum singular value of the Jacobian. $\lVert J \rVert_1$ and $\lVert J \rVert_{\infty}$ denotes $L_1$ and $L_{\infty}$ induced norm for the Jacobian. The mean value for 50 images is taken.

\begin{table*}[!t]
\caption{Summary of attacking results and corresponding estimated largest singular value. The attacks are employed for both ViTs and CoViTs with FGSM, PGD-7 and CW, and the robust accuracy are shown for each attack. The models with patch size $32\times 32$ are marked with $*$. $\lVert J \rVert_1$ and $\lVert J \rVert_{\infty}$ are the $L_1$ and $L_{\infty}$ norm respectively. The highest accuracy and lowest estimated maximum singular value are marked in bold.}
\label{tab:attacks_singular}
\vskip 0.15in
\begin{center}
\begin{small}
\begin{sc}
\begin{tabular}{l|l|l|l|l|l|l}
\toprule
             &Clean Acc. & FGSM & PGD-7($L_{\infty}$) & PGD-7($L_{2}$) & CW($L_{2}$) & $(\lVert J \rVert_1 \lVert J \rVert_{\infty})^{\frac{1}{2}}$  \\
\midrule
ViT-$S_1$ & $0.676$ & $0.213$      & $0.135$      & $0.267$      & $0.059$      & $812.69$  \\
ViT-$S_2$ & $0.739$ & $\bm{0.273}$ & $0.162$      & $\bm{0.348}$ & $0.067$      & $1003.20$ \\
CoViT-$S_1$ & $0.734$ & $0.254$    & $\bm{0.173}$ & $0.341$      & $\bm{0.144}$ & $242.78$ \\
CoViT-$S_2$ & $0.737$ & $0.244$    & $0.163$      & $0.328$      & $0.143$      & $\bm{206.78}$ \\
\midrule
ViT-$S_3$ & $0.847$ & $0.369$      & $0.221$      & $0.444$      & $0.053$ & $296.48$ \\
ViT-$S_4$ & $0.863$ & $\bm{0.392}$ & $\bm{0.240}$ & $\bm{0.448}$ & $0.065$ & $462.12$ \\
CoViT-$S_3$ & $0.882$ & $0.320$    & $0.179$      & $0.413$      & $\bm{0.104}$ & $\bm{146.48}$ \\
CoViT-$S_4$ & $0.876$ & $0.306$    & $0.170$      & $0.401$      & $0.088$ & $150.02$ \\
CoViT-$S_5$ & $0.868$ & $0.341$    & $0.192$      & $0.424$      & $0.082$ & $163.50$  \\
\midrule
ViT-$M_1$ & $0.877$ & $0.415$      & $0.267$      & $0.467$       & $0.049$ & $236.21$  \\
ViT-$M_2$ & $0.861$ & $0.415$      & $0.260$      & $0.461$       & $0.053$ & $294.06$ \\
*ViT-$M_3$ & $0.853$ & $\bm{0.478}$ & $\bm{0.356}$ & $\bm{0.519}$ & $0.103$ & $139.23$ \\
CoViT-$M_1$ & $0.881$ & $0.336$    & $0.185$      & $0.422$       & $0.051$ & $93.21$ \\
CoViT-$M_2$ & $0.882$ & $0.337$    & $0.197$      & $0.417$       & $0.086$ & $109.79$ \\
CoViT-$M_3$ & $0.870$ & $0.337$    & $0.194$      & $0.424$       & $0.072$ & $131.94$ \\
CoViT-$M_4$ & $0.875$ & $0.357$    & $0.208$      & $0.427$       & $0.093$ & $99.57$ \\
*CoViT-$M_5$ & $0.861$ & $0.416$    & $0.303$      & $0.480$      & $\bm{0.152}$ & $\bm{78.70}$ \\
\midrule
*ViT-$L$     & $0.848$ & $0.461$      & $0.347$      & $0.499$      & $0.094$ & $111.38$  \\
*CoViT-$L_1$ & $0.867$ & $0.443$      & $0.333$      & $0.505$      & $\bm{0.140}$ & $\bm{59.54}$ \\
*CoViT-$L_2$ & $0.853$ & $\bm{0.466}$ & $\bm{0.357}$ & $\bm{0.528}$ & $0.096$ &  $37.26$ \\
\bottomrule
\end{tabular}
\end{sc}
\end{small}
\end{center}
\vskip -0.1in
\end{table*}

The \textit{Robust Accuracy} for both ViTs and CoViTs attacked by FGSM, PGD-7 and CW is shown in Table \eqref{tab:attacks_singular}. For better comparison, we set $\epsilon = 2/225$ for both FGSM and PDG attack with $L_{\infty}$ norm. The step size for $L_{\infty}$ PGD attack is set to $\alpha = 2/255$ and it is iterated only for $7$ times to represent the weak attack. The $L_2$ norm PGD-7 is parameterized by $\epsilon = 2, \alpha = 0.2$. The parameters set for stronger CW attack is that $c=1$, adversarial confidence level $kappa=0$, learning rate for Adam \cite{kingma2014adam} optimizer in CW is set to $0.01$ and the total iteration number is set to $100$.      

As is shown in Table \eqref{tab:attacks_singular}, for weak attacks, i.e., FGSM and PGD-7, ViTs are generally exhibiting higher robust accuracy within the same group of similar parameter sizes with only a few exceptions. Also, both for ViTs and CoViTs, the robustness is strengthened as the model becomes deeper with more parameters.   

For a stronger CW attack, the result is almost reversed, CoViT model shows significantly better robustness and agrees with the approximation of the maximum singular value for the Jacobian. The ability of Self-Attention to avoid perturbed pixels is compromised as the attacking becomes stronger. And it seems that the \textit{translation invariance} of CNNs has more defensive power against strong attacks. In addition, a larger patch size always induces better adversarial robustness for both ViTs and CoViTs.

\section{Conclusion}

This paper first proves that ViTs are Lipschitz continuous for vision tasks, then we formally bridge up the local robustness of transformers with the Cauchy problem. We theoretically proved that the maximum singular value determines local robustness for the Jacobian of each block. Both small-scale and large-scale experiments have been conducted to verify our theories. With the proposed framework, we open the black box of ViTs and study how robustness changes among layers. We found that the first and last layers impede the robustness of ViTs. In addition, unlike existing research that argues MSA could boost robustness, we found that the defensive power of MSA in ViT only works for the large model under weak adversarial attacks. MSA even compromises the adversarial robustness under strong attacks.

\section{Discussion and Limitation}
The major limitations in this paper are embodied by the several approximations involved. The first one is the approximation of the underlying ODEs to the forward propagation of neural networks with a residual addition structure. As is shown in Lemma~\eqref{lemma_approx}, the approximation is accurate only when the neural networks are deep enough, and it is hard to know what depth is enough, given the required error bound. One possible way to make it accurate is to consider the \textit{Difference Equation}, which is a discrete parallel theory to ODEs. The second one is the approximation of the second-order term in Equation~\eqref{eq:thm1_6}. For small-size inputs, we can say that the $L_2$-norm of perturbations of adversarial examples is smaller enough so that the second term is negligible. However, the larger inputs may inflate the $L_2$-norm of perturbations since simply up sampling could result in a larger $L_2$-norm. Therefore, including the second term or choosing a better norm should be considered. The third approximation is shown in Equation~\eqref{eq:single_value_approx}. Since the size of the Jacobian depends on the size of the input image, making it impossible to directly calculate the singular value of the Jacobian for larger images, hence, we use an upper bound instead, which inevitably compromises the validation of the experiment. Moreover, since the adversarial attack can only get the upper bound of the minimal perturbations, it is also an approximation of the local robustness, as shown in Table~\eqref{tab:attacks_singular}. 

In the experimental part, we only take into account for the small to moderate size models because it is necessary to rule out the influence of pre-training, and we have to admit that the calculation for the singular value of Jacobian w.r.t. inputs of too large size is hardly implemented.

%\paragraph{Sample Heading (Fourth Level)}
%The contribution should contain no more than four levels of
%headings. Table~\eqref{tab1} gives a summary of all heading levels.

%\begin{table}
%\caption{Table captions should be placed above the
%tables.}\label{tab1}
%\begin{tabular}{|l|l|l|}
%\hline
%Heading level &  Example & Font size and style\\
%\hline
%Title (centered) &  {\Large\bfseries Lecture Notes} & 14 point, bold\\
%1st-level heading &  {\large\bfseries 1 Introduction} & 12 point, bold\\
%2nd-level heading & {\bfseries 2.1 Printing Area} & 10 point, bold\\
%3rd-level heading & {\bfseries Run-in Heading in Bold.} Text follows & 10 point, bold\\
%4th-level heading & {\itshape Lowest Level Heading.} Text follows & 10 point, italic\\
%hline
%\end{tabular}
%\end{table}

%\noindent Displayed equations are centered and set on a separate
%line.
%\begin{equation}
%x + y = z
%\end{equation}
%Please try to avoid rasterized images for line-art diagrams and
%schemas. Whenever possible, use vector graphics instead (see
%Fig.~\eqref{fig1}).

%
% the environments 'definition', 'lemma', 'proposition', 'corollary',
% 'remark', and 'example' are defined in the LLNCS documentclass as well.
%

\bibliographystyle{splncs04}
%\bibliography{main}
%

\end{document}

\title{Supplementary Material - Understanding Adversarial Robustness of Vision Transformers via Cauchy Problem\footnote{W. Ruan is the corresponding author.\\ W. Ruan is supported by Partnership Resource Fund (PRF) on Towards the Accountable and Explainable Learning-enabled Autonomous Robotic Systems from UK EPSRC project on Offshore Robotics for Certification of Assets (ORCA) [EP/R026173/1].}}
\titlerunning{Understanding Adversarial Robustness of Vision Transformers via Cauchy Problem}
\author{Zheng Wang\inst{1}\orcidID{0000-0001-7146-7503} \and
Wenjie Ruan \Letter \inst{1}\orcidID{0000-0002-8311-8738}}
\tocauthor{Zheng Wang (University of Exeter),
Wenjie Ruan (University of Exeter)}
\authorrunning{Z. Wang et al.}

\institute{University of Exeter, Exeter EX4 4PY, UK\\
\email{\{zw360, W.Ruan\}@exeter.ac.uk}}

\toctitle{Understanding Adversarial Robustness of Vision Transformers via Cauchy Problem}

\maketitle              % typeset the header of the contribution
\section*{Appendix A. Detailed Proofs}
\subsection*{Mean Value Inequality}
Let $D \subseteq \mathbb{R}^n$, and $\bm{f}: D \to \mathbb{R}^m$ which has partial derivatives on $D$ for each component. Hence $\forall \bm{a}, \bm{b} \in D$, there exists $\bm{\xi}$ on the line defined by $\bm{a}$ and $\bm{b}$, such that 
\begin{equation*}
    \lVert \bm{f}(\bm{b}) - \bm{f}(\bm{a})\rVert_2 \leq \lVert J_{\bm{f}}(\bm{\xi})\rVert_2 \lVert \bm{b} - \bm{a}\rVert_2
\end{equation*}
\begin{proof}
    Let 
\begin{equation*}
    \bm{r}(t) = \bm{\bm{a}} + t(\bm{b}-\bm{a}), t\in [0, 1]  
\end{equation*}
and 
\begin{equation*}
    \bm{g}(t) = \bm{f}\circ \bm{r}(t), t\in [0, 1]    
\end{equation*}
hence we got the mapping $\bm{g}: [0, 1] \to \mathbb{R}^m$, then define 
\begin{equation*}
    \phi(t) = \bigg(\bm{g}(1)-\bm{g}(0)\bigg)^T\bm{g}(t), t \in [0, 1]
\end{equation*}
since $\bm{f}$ is differentiable, $\phi$ is also differentiable. By applying \textit{mean value theorem}, there exists $\tau \in [0, 1]$ such that 

\begin{equation*}
    \phi(1) - \phi(0) = \phi(\tau)^{'} = \bigg(\bm{g}(1)-\bm{g}(0)\bigg)^TJ_{\bm{g}}(\tau),
\end{equation*}
because 
\begin{equation*}
    \phi(1) - \phi(0) = \lVert \bm{g}(1) - \bm{g}(0) \rVert_2^2,
\end{equation*}
we have 
\begin{equation*}
\begin{split}
    \lVert \bm{g}(1) - \bm{g}(0) \rVert_2^2 = \bigg(\bm{g}(1) - \bm{g}(0)\bigg)^TJ_{\bm{g}}(\tau),    
\end{split}
\end{equation*}
by applying Cauchy-Schwarz inquality, we have 
\begin{equation*}
\begin{split}
    \lVert \bm{g}(1) - \bm{g}(0) \rVert_2^2 &\leq \lVert \bm{g}(1) - \bm{g}(0) \rVert_2 \lVert J_{\bm{g}}(\tau)\rVert_2   \\ 
    \lVert \bm{g}(1) - \bm{g}(0) \rVert_2 &\leq \lVert J_{\bm{g}}(\tau)\rVert_2,   \\ 
\end{split}
\end{equation*}
Consider the derivative of $g$ w.r.t. $t$, that is  
\begin{equation*}
    J_{\bm{g}}(t) = J_{\bm{f}}(\bm{r}(t))J_{\bm{r}}(t) = J_{\bm{f}}(\bm{r}(t))(\bm{b}-\bm{a}),
\end{equation*}
therefore
\begin{equation*}
\begin{split}
    \lVert \bm{g}(1) - \bm{g}(0) \rVert_2 &\leq \lVert J_{\bm{g}}(\tau) \rVert_2 \\
    \lVert \bm{f}(b) - \bm{f}(a) \rVert_2 &\leq \lVert J_{\bm{f}}(\bm{r}(t))(\bm{b}-\bm{a}) \rVert_2 \\
                                &\leq \lVert J_{\bm{f}}(\bm{r}(t)) \rVert_2 \lVert \bm{b}-\bm{a} \rVert_2,
\end{split}
\end{equation*}
where $\bm{\xi} = \bm{r}(\tau)$ is the point on the line defined by $\bm{a}, \bm{b}$, and $\lVert J_{\bm{f}}(\bm{r}(t)) \rVert_2$ is the $L_2$ induced norm for the Jacobian. 
\end{proof}

\subsection*{The Computation of Jacobian}

Since the derivatives of the vectors even matrix are frequently considered, and for simplicity concerns, we take all vectors as vertical vectors, i.e., size of $n\times 1$ and denote the Jacobian of the function $\bm{f}:\mathbb{R}^{n} \to \mathbb{R}^{m}$ as 
$$\frac{\partial \bm{f(x)}}{\partial \bm{x}^T}$$, 
hence for $\bm{g}:\mathbb{R}^{m} \to \mathbb{R}^{s}$, we have 
$$
J_{\bm{f} \circ \bm{g}} = \frac{\partial \bm{f(x)}}{\partial \bm{g (x)}^T} \frac{\partial \bm{g(x)}}{\partial \bm{x}^T}.
$$  

\subsection*{The Jacobian of Non-linear Part for MSA (Special Case)} 

Consider the mapping $\bm{f}: \mathbb{R}^{N} \to \mathbb{R}^{N}$ and let 

\begin{equation*}
        \bm{f}(\bm{x}) = softmax(a\bm{xx}^T)\bm{x} = \bm{Px} =  \begin{pmatrix}
            \vdots \\
            \bm{p}(ax_i\bm{x}^T)\bm{x} \\
            \vdots
        \end{pmatrix}, \\    
\end{equation*}
therefore the Jacobian of $f$ w.r.t. $\bm{x}$ is
\begin{equation*}
    \begin{split}
        \frac{\partial \bm{f}}{\partial \bm{x}^T} &= \begin{pmatrix}
            \vdots \\
            \frac{\partial \bm{p}(ax_i\bm{x}^T)\bm{x}}{\bm{x}^T} \\
            \vdots
        \end{pmatrix} \\
        &= \begin{pmatrix}
            \vdots \\
            a\bm{x}^T\frac{\partial \bm{p}(\bm{u}_i)}{\partial \bm{u}_i^T}(x_i\bm{I} + \bm{x}\bm{e}_i^T) + \bm{p}(\bm{u}_i^T)\\
            \vdots
        \end{pmatrix} \\
        &= \begin{pmatrix}
            \vdots \\
            a\bm{x}^T(diag(\bm{p}(\bm{u}_i))-\bm{p}(\bm{u}_i)\bm{p}(\bm{u}_i^T))(x_i\bm{I} + \bm{x}\bm{e}_i^T)\\
            \vdots
        \end{pmatrix} + \bm{P} 
    \end{split},
\end{equation*}

where $\bm{p}(ax_i\bm{x}^T) = \bm{p}(\bm{u}_i^T)$ and $\bm{e}_i^T = (0,...,1,...0)$ which is a row vector with $i_{th}$ entry of $1$. Since each row of $\bm{P}$, i.e., $\bm{p}(\bm{u}_i^T)$, has all positive entries and $\sum\limits_{i=1}^{N}p_i(\bm{u}_i^T) = 1$, we regard it as one set of probability for $x_i, i=1,...,N$. Hence there are $N$ sets of probability for elements of $\bm{x}$. We define the \textit{mean} of each corresponding set of probability as  

\begin{equation*}
    \mu_i = \sum\limits_{i=1}^{N}p_i(\bm{u}_i^T)x_i, \quad i=1,2,...,N
\end{equation*}

naturally we have $\bm{\mu} = \bm{Px}$ and the \textit{variance} is defined as 

\begin{equation*}
\bm{\sigma}^2 = \begin{pmatrix}
                \sum_{k=1}^{N}p_k(\bm{u}_1)x_k^2 - \big(\sum_{k=1}^{N}p_k(\bm{u}_1)x_k\big)^2 \\
                \vdots \\
                \sum_{k=1}^{N}p_k(\bm{u}_N)x_k^2 - \big(\sum_{k=1}^{N}p_k(\bm{u}_N)x_k\big)^2 \\
\end{pmatrix} = \begin{pmatrix}
                \sigma_1^2 \\
                \vdots \\
                \sigma_N^2
\end{pmatrix}.    
\end{equation*}

Hence, 

\begin{equation*}
    \begin{split}
        &\begin{pmatrix}
        \vdots \\
        a\bm{x}^T(diag(\bm{p}(\bm{u}_i))-\bm{p}(\bm{u}_i)\bm{p}(\bm{u}_i^T))(x_i\bm{I} + \bm{x}\bm{e}_i^T)\\
        \vdots
        \end{pmatrix} \\
        &= a\begin{pmatrix}
        \vdots \\
        x_i\bm{x}^Tdiag(\bm{p}(\bm{u}_i)) + \bm{x}^Tdiag(\bm{p}(\bm{u}_i))\bm{x}\bm{e}_i^T -x_i\bm{x}^T\bm{p}(\bm{u}_i)\bm{p}(\bm{u}_i^T) - \bm{x}^T\bm{p}(\bm{u}_i)\bm{p}(\bm{u}_i^T)\bm{x}\bm{e}_i^T \\
        \vdots
        \end{pmatrix}\\
    \end{split}
\end{equation*}

where 
\begin{equation*}
x_i\bm{x}^Tdiag(\bm{p}(\bm{u}_i)) = \begin{pmatrix}
x_ix_1p_1(\bm{u}_i) & \cdots & x_ix_Np_N(\bm{u}_i)
\end{pmatrix}
\end{equation*}

\begin{equation*}
\bm{x}^Tdiag(\bm{p}(\bm{u}_i))\bm{x}\bm{e}_i^T = \begin{pmatrix}
0 & \cdots & \sum\limits_{k=1}^{N}p_k(\bm{u}_i)x_k^2 & \cdots & 0
\end{pmatrix}
\end{equation*}

\begin{equation*}
x_i\bm{x}^T\bm{p}(\bm{u}_i)\bm{p}(\bm{u}_i^T) = \begin{pmatrix}
x_i\mu_ip_1(\bm{u}_i^T) & \cdots & x_i\mu_ip_N(\bm{u}_i^T)
\end{pmatrix}
\end{equation*}

\begin{equation*}
\bm{x}^T\bm{p}(\bm{u}_i)\bm{p}(\bm{u}_i^T)\bm{x}\bm{e}_i^T = \begin{pmatrix}
0 & \cdots & \mu_i^2 &\cdots & 0
\end{pmatrix}.
\end{equation*}

Hence we have

\begin{equation*}
    \begin{split}
            \frac{\partial \bm{f}}{\partial \bm{x}^T} &=a\begin{pmatrix}
            & & \vdots\\
            (x_1-\mu_i)p_1(\bm{u}_i)x_i & \cdots & (x_i-\mu_i)p_i(\bm{u}_i)x_i & \cdots & (x_N-\mu_i)p_N(\bm{u}_i)x_i \\
            & & \vdots \\
            \end{pmatrix} \\
            &+ adiag(\bm{\sigma}^2) + \bm{P} \\
        &=a\big\{diag(\bm{x})Pdiag(\bm{x}) - Pdiag(\bm{x})diag(\bm{\mu}) + diag(\bm{\sigma}^2) \big\} + \bm{P}
    \end{split}
\end{equation*}

\subsection*{The Jacobian of non-linear part for MSA (General Case)} 

Define the mapping $f: \mathbf{R}^{N\times D} \to \mathbf{R}^{N\times D}$ and let 

\begin{equation*}
    f(X) = softmax(XA^TX^T)X.    
\end{equation*}

Since $f(X)$ is a matrix, each row of $f(X)$ can be defined as 
    
\begin{equation*}
      f_i(X)^T = f_{i,\cdot}(X) = softmax(x^T_iA^TX^T)X.  
\end{equation*}

Therefore 

\begin{equation*}
    \begin{split}
    f_i(X) &= X^Tsoftmax(XAx_i)  \\
    &= X^Tsoftmax(u_i)  \\
    &= X^Tp(u_i) \\
    &= \sum_{k=1}^{N}x_kp_k(u_i), 
    \end{split}
\end{equation*}
    
where 
\begin{equation*}
u_i = XAx_i = X_{N\times D}A_{D\times D}x_{i, D\times 1} = 
\begin{pmatrix}
x_1^TAx_i \\
x_2^TAx_i \\
\vdots \\
x_N^TAx_i \\
\end{pmatrix} \in \mathbf{R}^{N\times 1}
\end{equation*}
and $p: \mathbf{R}^N\to \mathbf{R}^N$ denotes the function of softmax. Then the Jacobian of $f_i(X)$ w.r.t. $x_j$ is 

\begin{equation*}
\begin{split}
     \frac{\partial f_i}{ \partial x_j^T} &= \sum_{k=1}^{N}\frac{\partial x_kp_k(u_i)}{ \partial x_j^T} \\
     &= \sum_{k=1}^{N}x_k \frac{\partial p_k(u_i)}{\partial x_j^T} + E_{D\times D}p_j(u_i) \\
     &= \sum_{k=1}^{N}x_k \frac{\partial p_k(u_i)}{\partial u_i^T}\frac{\partial u_i}{x_j^T} + E_{D\times D}p_j(u_i) \\
     &= \sum_{k=1}^{N}x_k \frac{\partial p_k(u_i)}{\partial u_i^T}\frac{\partial XAx_i}{x_j^T} + E_{D\times D}p_j(u_i) \\
     &= \sum_{k=1}^{N}x_k \frac{\partial p_k(u_i)}{\partial u_i^T}\frac{\partial XAx_i}{x_j^T} + E_{D\times D}p_j(u_i) \\
     &= X^T\frac{\partial p(u_i)}{\partial u_i^T} \frac{\partial XAx_i}{x_j^T} + EP_{ij},
\end{split}
\end{equation*}
where 
\begin{equation*}
\frac{\partial XAx_i}{x_j^T} = \frac{\partial \begin{pmatrix}
x_1^TAx_i \\
x_2^TAx_i \\
\vdots \\ 
x_N^TAx_i \\
\end{pmatrix}}{\partial x_j^T},    
\end{equation*} 
and $P_{ij}$ is the ${ij}_{th}$ element of $softmax(XA^TX^T)$. In order to have the gradient above, consider following

\begin{equation*}
\begin{split}
     \frac{\partial x^TAx}{\partial x^T} &= x^T(A^T + A)  \\
     \frac{\partial y^TA x}{\partial x^T} &= y^TA \\
     \frac{\partial x^TA y}{\partial x^T} &= y^TA^T     
\end{split}
\end{equation*}
Hence,
\begin{equation*}
\frac{\partial XAx_i}{x_j^T} = \begin{pmatrix}
x_1^T  \\
x_2^T  \\
\vdots \\
x_j^T  \\
\vdots \\
x_D^T 
\end{pmatrix}A\delta_{ij} + \begin{pmatrix}
0_D  \\
0_D  \\
\vdots \\
x_i^T \\
\vdots \\
0_D
\end{pmatrix}A^T = XA\delta_{ij} + E_{ji, N \times N}XA^T
\end{equation*}.
In a nut, 

\begin{equation*}
    \frac{\partial f_i}{ \partial x_j^T} =  X^T\frac{\partial p(u_i)}{\partial u_i^T}\bigg[XA\delta_{ij} + E_{ji, N \times N}XA^T\bigg] + E_{D\times D}P_{ij},
\end{equation*}

where $\frac{\partial p(u)}{\partial u^T} = diag(p(u))-p(u)p(u)^T$ and the Jacobian of $f$ w.r.t. input $X$ is 
\begin{equation*}
     Jf(X) = \bigg( \frac{\partial f_i}{ \partial x_j^T} \bigg)_{ND \times ND}  \\
\end{equation*}
It is obviously that each component in equation above are bounded since inputs are bounded. We can conclude that the \textit{Frobenius norm} of The Jacobian is bounded. Since all the matrix norms are equvilent in finite dimension. The Jacobian is also bounded with $L_2$ induced norm. 

\section*{Appendix B. Configuration of Models and Training Result}

\begin{table}[!htbp]
\caption{Summary of Configuration and Training Results. The models are grouped based on their parameter size and the depth of the model. The models within the same group have the same depth and similar parameter size. We have 4 groups corresponds to the depth of 1, 4, 8 and higher. It is obvious that training results of models within the same group have similar performance.}
\label{tab:config_train}
\vskip 0.15in
\begin{center}
\begin{small}
\begin{sc}

\begin{tabular}{l|l|l|l|l|l|l}
\toprule
          & \#Head & Depth & Patch & Par.(MB) & Train Acc.(Loss) & Test Acc.(Loss)\\
\midrule
ViT-$S_1$ & 1     &1      & $16 \times 16$ & 13.55 & 0.698 (0.860) & 0.676 (0.918) \\
ViT-$S_2$ & 4     &1      & $16 \times 16$ & 13.55 & 0.768 (0.662) & 0.739 (0.745)\\
CoViT-$S_1$ & K3             &1& $16 \times 16$ & 13.55 & 0.754 (0.706) & 0.734 (0.756)  \\
CoViT-$S_2$ & 4 $\times$ k3  &1& $16 \times 16$ & 11.30 & 0.756 (0.703) & 0.737 (0.754)\\
\midrule
ViT-$S_3$ & 1     &4      & $16 \times 16$ & 49.63 & 0.976 (0.078) & 0.847 (0.525) \\
ViT-$S_4$ & 4     &4      & $16 \times 16$ & 49.63 & 0.996 (0.018) & 0.863 (0.571)\\
CoViT-$S_3$ & k3             &4& $16 \times 16$ & 49.61 & 0.968 (0.105) & 0.882 (0.384) \\
CoViT-$S_4$ & 4 $\times$ k3  &4& $16 \times 16$ & 40.61 & 0.970 (0.094) & 0.876 (0.405) \\
CoViT-$S_5$ & 4 $\times$ k7  &4& $16 \times 16$ & 44.61 & 0.992 (0.032) & 0.868 (0.491) \\
\midrule
ViT-$M_1$ & 1     &8      & $16 \times 16$ & 97.73 & 0.999 (0.004) & 0.877 (0.514) \\
ViT-$M_2$ & 4     &8      & $16 \times 16$ & 97.73 & 0.999 (0.002) & 0.861 (0.641)\\
ViT-$M_3$ & 4     &8      & $32 \times 32$ & 102.23& 0.997 (0.013) & 0.853 (0.635) \\
CoViT-$M_1$ & k3             &8& $16 \times 16$ & 97.70 & 0.986 (0.054) & 0.881 (0.407) \\
CoViT-$M_2$ & k1357          &8& $16 \times 16$ & 81.70 & 0.996 (0.021) & 0.882 (0.471) \\
CoViT-$M_3$ & 4 $\times$ k3  &8& $16 \times 16$ & 79.70 & 0.973 (0.081) & 0.870 (0.431)\\
CoViT-$M_4$ & 4 $\times$ k7  &8& $16 \times 16$ & 87.70 & 0.997 (0.015) & 0.875 (0.513)\\
CoViT-$M_5$ & 4 $\times$ k3  &8& $32 \times 32$ & 84.20 & 0.965 (0.108)& 0.861 (0.459) \\
\midrule
ViT-$L$ & 8     &12      & $32 \times 32$ & 150.33 & 0.999 (0.003) & 0.848 (0.681) \\
\multirow{2}{*}{CoViT-$L_1$} & 4$\times$ k3 &\multirow{2}{*}{12} & \multirow{2}{*}{$32 \times 32$} & \multirow{2}{*}{120.28} & \multirow{2}{*}{0.979 (0.068)} & \multirow{2}{*}{0.867 (0.467)} \\
                            & 4$\times$ k5 &                    &  & & &\\
\multirow{2}{*}{CoViT-$L_2$} & 4$\times$ k3 &\multirow{2}{*}{16} & \multirow{2}{*}{$32 \times 32$} & \multirow{2}{*}{222.03} & \multirow{2}{*}{0.954 (0.130)} & \multirow{2}{*}{0.853 (0.497)} \\
                            & 4$\times$ k5 &                    &  & & &\\
\bottomrule
\end{tabular}
\end{sc}
\end{small}
\end{center}
\vskip -0.1in
\end{table}

As is shown in table \ref{tab:config_train}, all the large models are divided into 3 groups suffixed as $S_i$, $M_i$, $L_i$, standing for 'Small', 'Medium' and 'Large', which correspond to their parameter size represented by \textit{Par.} in both tables, and $i$ is different for ViT and CoViT, denoting the parallel model of similar size. It is worth noting that, in table \ref{tab:config_train}, the parameters to be trained will reduce as the number of heads increases due to the way how 1-D convolution layers are concatenated. For example, the parameter size for CoViT-$S_1$ is $512\times 512 \times 3 + 512$ and that for CoViT-$S_2$ is $4 \times (128 \times 128 \times 3 + 128)$ which is much smaller. 

For both ViTs and CoViTs, the performance of models within the same group are similar to each other and the shallow networks, e.g., ViT-$S_1$, ViT-$S_2$, CoViT-$S_1$, CoViT-$S_2$, are harder to train and may need extra training to be converged. This may due to the optimizer used, i.e., SAM, since SAM will try to find shallow-wide optima instead of deep-narrow one, which requires stronger model capacity. The reason of use SAM as optimizer instead of Adam is due to sharp loss surface for ViTs, which will result in large generalization gap if not pretrained on larger dataset. For a better comparison, both ViTs and CoViTs are trained by SAM with the same hyper-parameter for 150 epochs of which the training result are reported.